\documentclass[a4paper]{article}
    \usepackage[margin=25mm]{geometry}
    \usepackage{amsmath}
    \usepackage{amsthm}
    \usepackage{amsfonts}
    \usepackage{amssymb}
    \usepackage{graphicx}
    \usepackage[fontsize=14pt]{scrextend}
    \usepackage{verbatim}
    \usepackage{tikz-cd}
    \usepackage{hyperref}

    \newtheorem{definition}{Definition}[section]
    \newtheorem{theorem}{Theorem}

    \newtheorem{corollary}[theorem]{Corollary}

    \newtheorem{example}{Example}[definition]
    
    \newtheorem{case}{Case}[section]
    
    \immediate\write18{texcount -tex -sum  \jobname.tex > \jobname.wordcount.tex}

    \providecommand{\keywords}[1]
    {
      \small	
      \textbf{\textit{Keywords---}} #1
    }
    
    \title{A novel approach for solving a variant of Transportation Problem}
    \author{Swapnil K Sinha$^{1}$, Sasikanth Raghava Goteti$^{2}$  \\
            \small $^{1}$swapnil.bit@iitbombay.org \\
            \small $^{2}$raghavas@thoughtworks.com \\
    }
    \date{} 
    
\begin{document}
    \maketitle
    
    \begin{abstract}
        In this article we consider a certain sub class of Integer Equal Flow problem, which are known NP hard \cite{meyers}. 
        Currently there exist no direct solutions for the same. It is a common problem in various inventory management systems.
        Here we discuss a local minima solution which uses projection of the convex spaces to resolve the equal flows and turn the problem into a known linear integer programming or constraint satisfaction problem which have reasonable known solutions and can be effectively solved using simplex or other standard optimization strategies.
       \end{abstract}
     
    \maketitle
    
    \section{Problem Space}
    Integer equal flow problems are known to be NP-Hard as observed by Meyers and Schulz \cite{meyers}. Most solutions to this would
    require graph theoretic language to formulate it correctly, as proposed by Morrison et al. \cite{Morrison2013ANS}. Effective
    algorithms like network simplex can be used to iteratively improve upon a simple feasible solution. We can formally define an \textit{integer equal flow problem} as:
    \begin{equation}
        \begin{array}{rrclcl}
        \displaystyle \min_{X} & \multicolumn{3}{l}{c^T X} \\
        \textrm{s.t.} & \sum_{j:(i,j) \in A } x_{ij} & - & \sum_{j:(j,i) \in A } x_{ji} & = & b_i \\
        & x_{ij} & \geq & 0 & & \forall j \in N \\
        & x_{ij} & \leq & u_{ij} \\
        & x_{ik} &= & t && \forall {ik}\in Class(t) \\
        & x_{ij} & \in & Z
        \end{array}
        \end{equation}

    One could extend the same definition of network flow optimization to specific cases of transportation and assignment problems. An \textit{equal transportation problem} can be extended as:
        \begin{equation}
        \begin{array}{rrclcl}
        \displaystyle \min_{X} & \multicolumn{3}{l}{c^T X} \\
        \textrm{s.t.} & \sum_{j=1}^n x_{ij}  & = & b_i \\
        & \sum_{i=1}^n x_{ij}  & = & a_j \\
        & x_{ij} & \geq & 0 & &  \\
        & x_{ij} & \leq & u_{ij} \\
        & x_{ik} &= & t && \forall {ik}\in Class(t) \\
        & x_{ij} & \in & Z
        \end{array}
        \end{equation}
        
    In many cases of transportation or assignment problems, it is important to follow the same routes or same inventory allocations over a period of repeated inventory assignment cycles mainly to reduce maintenance costs or other taxation charges. To resolve these scenarios, it is important that the same inventory takes the same route every time. So if we add a third index  $\tau$ for time period, we can define an equal assignment over time. We define a corresponding \textit{same route transportation} problem as:

     \begin{equation}
        \begin{array}{rrclcl}
        \displaystyle \min_{X} & \multicolumn{3}{l}{c^T X} \\
        \textrm{s.t.} & \sum_{j=1}^n x_{ij\tau}  & = & b_{i\tau} \\
        & \sum_{i=1}^n x_{ij\tau}  & = & a_{j\tau} \\
        & x_{ij\tau} & \geq & 0 & & \\
        & x_{ij\tau} & \leq & u_{ij\tau} \\
        & x_{ik\tau} &= & t_{ik} && \forall {\tau}\in Class(t_{ik}) \\
        & x_{ij\tau} & \in & Z
        \end{array}
        \end{equation}

    \theoremstyle{plain}
    \begin{definition}{Class:}
    We define the set of all arcs in a same transportation with same value and same indexes as $Class(t_{ik\tau})$ 
    where indices are represented as usual interpretation
    \end{definition}
    \begin{example}
    $Class(t_{i,,})$ represents set of all source vectors that are equal to $t_{i}$
    \end{example}
    \begin{example}
    $Class(t_{,k,})$ represents set of all destination vectors that are equal to $t_{k}$
    \end{example}
    \begin{example}
    $Class(t_{i,k,})$ represents set of all arc scalars that are equal to $t_{ik}$ across all $\tau$
    \end{example}
    \begin{example}
    $Class(t_{,,})$ represents set of all transportation problems that have same solution space across all $\tau$
    \end{example}
    
    Often in usage we will ignore the $,$ and leave it for interpretation with the subscript used. One can extend the same problem to a much broader setting of assignment problem. A typical \textit{same inventory assignment problem} would look like:

        \begin{equation}
        \begin{array}{rrclcl}
        \displaystyle \min_{X} & \multicolumn{3}{l}{c^T X} \\
        \textrm{s.t.} & \sum_{j=1}^n x_{ij\tau}  & = & 1 \\
        & \sum_{i=1}^n x_{ij\tau}  & = & 1 \\
        & x_{ij\tau} & \in &\{0,1\} \\
        & x_{ik\tau} &= & t_{ik} && \forall {\tau}\in Class(t_{ik})
        \end{array}
        \end{equation}
        
    One could formulate generalized assignment and generalized transportation problems under the same breadth. We will specifically deal with a generalized \textit{same route assignment problem}, which can be formulated as:
    
        \begin{equation}
        \begin{array}{rrclcl}
        \displaystyle \min_{X} & \multicolumn{3}{l}{c^T X} \\
        \textrm{s.t.} & \sum_{j=1}^n x_{ij\tau}  & = & b_{i\tau} \\
        & \sum_{i=1}^n x_{ij\tau}  & = & 1 \\
        & x_{ij\tau} & \in &\{0,1\} \\
        & x_{ik\tau} &= & t_{ik} && \forall {\tau}\in Class(t_{ik})
        \end{array}
        \end{equation}

    \section{Feasibility Certificates}
    Feasibility certificate or a gale certificate \cite{gale1957} typically can be understood as determining if the constraints in a network optimization have a feasible flow or not. This is very easy to determine in a simple transportation problem - if it's balanced we are always guaranteed to have a feasible solution. One can quickly verify it using a simple north-west corner solution. For a more rigorous treatment of the same we will need a Matroid theory. In specific to solving a forbidden arc transportation problem we will need to prove the existence of a Monge sequence as observed by Shamir et al. \cite{ADLER199321}. We can define a \textit{forbidden-arc same route assignment problem} as below.
    
    \begin{definition}
        We call the set of all arcs where the flow is constrained to void as forbidden arcs and represent this set by $\mathfrak{F}$
        $$\mathfrak{F}=\{(i,k,\tau) \mid x_{ik\tau}=0 \}$$
    \end{definition}
    
        \begin{equation}
        \begin{array}{rrclcl}
        \displaystyle \min_{X} & \multicolumn{3}{l}{c^T X} \\
        \textrm{s.t.} & \sum_{j=1}^n x_{ij\tau}  & = & b_{i\tau} \\
        & \sum_{i=1}^n x_{ij\tau}  & = & 1 \\
        & x_{ij\tau} & \in &\{0,1\} \\
        & x_{ik\tau} &= & t_{ik} && \forall {\tau}\in Class(t_{ik}) \\
        & x_{ij\lambda} & = & 0 && \forall {(i,j,\lambda)}\in \mathfrak{F}\\
        \end{array}
        \end{equation}
        
    where $\mathfrak{F}$ is the set of all forbidden arcs (note that the equal flow arcs are counted multiple times in the forbidden arc constraints $\mathfrak{F}$ as they are equally 0 everywhere).
    
    In a similar setting one could define the feasibility certificate problem of a forbidden-arc same route generalized assignment problem as:
    
        \begin{equation}
        \begin{array}{rrclcl}
        \displaystyle \exists &X \\
        \textrm{s.t.} & \sum_{j=1}^n x_{ij\tau}  & = & b_{i\tau} \\
        & \sum_{i=1}^n x_{ij\tau}  & = & 1 \\
        & x_{ij\tau} & \in &\{0,1\} \\
        & x_{ik\tau} &= & t_{ik} && \forall {\tau}\in Class(t_{ik}) \\
        & x_{ij\lambda} & = & 0 && \forall {(i,j,\lambda)}\in \mathfrak{F}\\
        \end{array}
        \end{equation}
        
    \section{Problem reformulations}
    \subsection{Stacking By index}
    One could reformulate the problem by stacking the individual transportation problems into one big transportation problem and setting or forcing all the irrelevant arcs forcibly to zero. Once this is done the transportation matrix would look like a block diagonal matrix with indexed transportation problems along the block diagonal and the entire matrix can itself be handled as a forbidden arc transportation problem.
    
    A similar transportation matrix would look like:
        \[
        \begin{matrix}
            x_{000} & x_{010} & 0 & 0 & \cdots & 0  \\
            x_{100} & x_{110} & 0 & 0 & \cdots & 0  \\
            0 & 0 & x_{001} & x_{011}& \cdots & 0  \\
            0 & 0 & x_{101} & x_{111} & \ddots & \vdots \\
            \vdots & \vdots & \vdots & \ddots & \ddots &  x_{01\tau}  \\
            0 & 0 & 0 & \cdots & x_{n0\tau} & x_{n1\tau}
        \end{matrix}
    \]
    
    \subsection{Resolving Sparsity}
    A diagonal transportation matrix like the one above has serious sparsity problems making it difficult to solve using conventional transportation strategies. Hence we will modify the cost function to throw high cost to get rid of the sparse arcs or force set them to 0s, so an equivalent form of a sparse transportation problem can be reinterpreted as:
    
        \begin{equation}
        \begin{array}{rrclcl}
        \displaystyle \min_{X} & \multicolumn{3}{l}{c^T X + \Lambda^TY } & &\forall y_{(i,j,\lambda)} \in \mathfrak{F}\\
        \textrm{s.t.} & \sum_{j=1}^n x_{ij}  & = & b_{i} \\
        & \sum_{i=1}^n x_{ij}  & = & 1 \\
        & x_{ij} & \in &\{0,1\} \\
        & x_{ik} &= & t_{ik} && \forall {ik}\in Class(t_{ik}) \\
        && \lambda >> 1
        \end{array}
        \end{equation}
        
    Note that we have converted the sparsity constraints into the objective by taking the dual form for those particular constraints.
    
    \subsection{Resolving equality constraints}
    One could further the same idea and even get rid of the equality conditions but at the cost of introducing a quadratic objective.
    One could reformulate the problem as
    
    \begin{equation}
        \begin{array}{rrclcl}
        \displaystyle \min_{X} & \multicolumn{3}{l}{c^T X + \Lambda^TY + 
        \lambda\sum_{\forall(i,k) \in Class(t_{ik})}(x_{ik}-t_{ik})^2
        } & &\forall {y_{(i,j,\lambda)}}\in \mathfrak{F}\\
        \textrm{s.t.} & \sum_{j=1}^n x_{ij}  & = & b_{i} \\
        & \sum_{i=1}^n x_{ij}  & = & 1 \\
        & x_{ij} & \in &\{0,1\} \\
        && \lambda >> 1
        \end{array}
        \end{equation}
        
    The above formulation is a standard generalized quadratic transportation problem which has some known approaches. One could use an equalization method as proposed by Marcin et al. \cite{Anholcer2015TheNG}.
    But this is still a quadratic transportation problem with no bounds on convergence that can be given.
    
    One could even stack the indexes in a multidimensional way and arrive at a 3D transportation problem which has some known solutions as proposed by Stefan et al. \cite{Stefan}. But even this approach cannot get rid of the quadratic cost function. One could have easy solutions to this problem once it has been linearized in some form as noted by Bertsekas et al. \cite{4739098}. However it is not possible to linearize a quadratic  transportation problem the way it could be done with an assignment problem using Glover Linearization \cite{GUEYE20091255}. So it is significantly harder to solve a quadratic transportation problem than even solving a quadratic assignment problem and the state of the art solutions for QAP almost never scale over $n=20$ \cite{BURKARD1991115}. In case of quadratic transportation problem its much harder to even formulate the complexity as a dependent of $n$.

    \section{Exploiting the inherent symmetries}
    \theoremstyle{definition}
    \begin{definition}{Similar Route:}
    Similar routes in a same route transportation problem are the row vector which are bounded by the equality constraint. It is represented as $\mu_k$. It is simply the arcs represented by indices in $Class(t_{i})$:
    $$\mu_k = \{x_{ik\tau} \mid (i,\tau) \in Class(t_{k})\}$$
    \end{definition}
    
    \begin{definition}{Equibounded:}
    A same route transportation problem is called equibounded if all the arcs in the similar routes have the same bounds, A same route transportation problem is equibounded if it satisfies the predicate:
    $$\exists l_{ik},\exists u_{ik}(l_{ik}<t_{ik}<u_{ik}, \forall x_{ik\tau} \in A_{ik} \in A_k \mid x_{ik\tau}  =  t_{ik})$$
    \end{definition}
    
        \begin{definition}{Symbol of A transportation:}
    The span of a same route transportation problem is the count of its sources, destinations and similarity classes. It is represented by $\varsigma_{ik\tau}$. It can be assigned a value $i*k*\tau$, called the size of a transportation. In general, a transportation problem is said feasible by its symbol $\varsigma_{ik\tau}$ directly.
    \end{definition}
    
    \begin{definition}{Symbol of An Assignment:}
    The span of a same route Assignment problem is the count of its sources, destinations and similarity classes. It is represented by $\alpha_{ik\tau}$. It can be assigned a value $i*k*\tau$, called the size of a Assignment. In general, an Assignment problem is said feasible by its symbol $\alpha_{ik\tau}$ directly.
    \end{definition}
    
    \subsection{Equivalence Theorems}
    \begin{theorem}[Weak Equivalence] \label{thm:we}
    For every same route Assignment problem with a feasibility certificate there exists a corresponding Same route transportation problem with a feasibility certificate whenever the problems are equibounded.
    \end{theorem}
    \begin{proof}
    Its very easy to see that for every transportation setup there exists an equivalent assignment problem that can be solved but to prove the equivalence we need to prove the converse as well.
    \newline
    Since the transportation problem and its equivalent Assignment problem are equibounded, we will assume that the lower bound is 0 and just prove for the case of upper bound. Similar proof can be extended to the case of double bound as well.
    \newline
    We conduct the proof by induction:
    \newline For $n=1$, if we have $x_{0i}+x_{0(i+1)}+... =b_i$, wherever we have $x_{ij}>0$ the corresponding assignment variables should all be set to 1 else 0. If the assignment is surplus then we can randomly sample any set of variables. If assignment is insufficient, then the transportation problem would have never given a feasibility certificate.
    \newline If for some $n=k$, ($\varsigma_{ik\tau} \iff \alpha_{ik\tau}) \wedge \varsigma_{i(k+1)\tau} $ then we need to prove that $\alpha_{i(k+1)\tau}$.
    \newline
    For the $(k+1)^{th}$ route take the $x_{i(k+1)}$ values from $\varsigma_{i(k+1)\tau}$ and randomly assign $1^s$ to corresponding variables and subtract these assigned values from both $\varsigma_{i(k+1)\tau}$ and $\alpha_{i(k+1)\tau}$. The remaining constraints correspondingly constitute $\varsigma_{i(k)\tau}$ and $\alpha_{i(k)\tau}$, which we know are feasible by inductive step definition.
    \newline Intuitively we can reduce every $\alpha_{i(k+1)\tau}$ into
    $\alpha_{i(k)\tau}$ and $\alpha_{i(1)\tau}$, both of which we know are feasible. 
    \end{proof}
    
    \begin{theorem}[Strong Equivalence] \label{thm:se}
    For every same route Assignment problem with a feasibility certificate there exists a corresponding Same route transportation problem with a feasibility certificate 
    \end{theorem}
    \begin{proof}
    In case of \textit{unequibound} problem we split the variables into two and realize that the equivalent problem is solvable for every variable in $\varsigma_{ik\tau}$. We just need to realize that there exists a $\varsigma_{i(k+k^\prime)\tau} \mid \varsigma_{ik\tau} \subset \varsigma_{i(k+k^\prime)\tau} \wedge \varsigma_{ik\tau}$ is equibound. We also need to reformulate a corresponding $\alpha_{i(k+k^\prime)\tau}$. One could also use direct induction like in equibound case. Detailed proof is left to the reader and is explained further with an example in section \ref{cs3}.
    \end{proof}
    
    \section{Category Theory}
    \begin{definition}{Covariant Functor:}
    If $\mathfrak{A}$ and  $\mathfrak{B}$ are two categories then a functor between them is a map that carries every arrow of the category $\mathfrak{A}$ to an arrow in $\mathfrak{B}$ between the same two objects such that identity and composition are preserved.
    \end{definition}   
    \begin{example} Category of open sets in a topology and boolean algebras representing the inclusion the base sets.
    \end{example}
    \begin{definition}{Contravariant Functor:}
    If $\mathfrak{A}$ and  $\mathfrak{B}$ are two categories then a functor between them i a map that carries every arrow of the category $\mathfrak{A}$ to an arrow in $\mathfrak{B}$ between the same two objects such that identity and composition are preserved and the direction of the arrows are reversed.
    \end{definition} 
    \begin{example} Presheafs in geometry
    \end{example}
    \begin{example}category of convex sets and category of sets of linear in-equations can have a contravariant functor that represents the set of all linear equations that contain the given hull. Inclusions are reversed
    \end{example}
    \begin{definition}{Natural transformation:}
    Natural transformation is a morphism that carries every arrow on the functors naturally such that the below diagram commutes \cite{maclane:71}:
    \end{definition}
    
    \begin{tikzcd}
    F(X) \arrow{d}{\eta_{x}} \arrow{r}{F(f)}
    & F(Y) \arrow{d}{\eta_{y}} \\
    G(X) \arrow{r}{G(f)}
    &G(Y) 
    \end{tikzcd}
    
    \section{Elimination Strategies}
    
    Fourier-Motzkin Elimination (FME) is one of the most common ways to assign feasibility certificates through variable elimination. However, it has limited practicality due to its double-exponential worst time complexity. Even the parallel deployment of FME can only linearly speed-up the process, for both dense and sparse problems \cite{Keßler1996}.
    One could also eliminate variables by realizing that for binary constraints we would never have more than $2^n$ constraints unlike in the case of Fourier-Motzkin where coefficients can be other than 0 or 1.
    So most constraints are going to be duplicates or just inclusions and the weaker constraints can be eliminated straight away. In fact one can rigorously prove that the constraint space would form a measure with values as integers. But for now we will omit that  and just propose the algorithm for generating the constraints in the section below which is all that is needed for the purpose of this article.
    We call this Elimination strategy Fourier-Binary-Constraint-Elimination (FBCE)\label{FBCE}.  
    
    \section{Solution Approaches}

    \begin{theorem} Variable elimination is a natural transformation of projection between convex spaces.
    \end{theorem}
    \begin{proof}
    Its clear to see that every set of linear equations can be mapped to a convex space by a contravariant functor. Every elimination morphism is an inclusion map in the category of convex sets hence it must have a contravariant inclusion.  Simply put if a convex set contains another, then its corresponding linear algebraic set must also have a morphism.
    The following diagram commutes where $h$ is the inclusion map on the set of linear equations (a kin to presheaf in algebraic geometry): 

    \begin{tikzcd}
    lin(X) \arrow{d}{\eta_{x}} \arrow{r}{lin(h)}
    & lin(Y) \arrow{d}{\eta_{y}} \\
    plin(X) \arrow{r}{plin(h)}
    &plin(Y) 
    \end{tikzcd}
    \end{proof}
    
    \begin{corollary}[Projection] \label{thm:pt}
    For every $\varsigma_{ik}$  there exists a projection in $\mathcal{O}({2^i})$ constraints and k variables.
    \end{corollary}
    \begin{proof}
    Elimination strategy {FBCE} (\ref{FBCE}) is a projection on the convex space,  hence the proof follows from above theorem.
    \end{proof}

    \begin{corollary}[Intersection] \label{thm:it}
    For every $\varsigma_{ik\tau}$  there exists a corresponding constraint satisfaction problem in $\mathcal{O}({2^i})$ constraints and k variables.
    \end{corollary}
    \begin{proof}
    Proof is trivial. Once we have established the natural transformation, all inclusions in case of linear equation, are trivial. This is similar to gluing axiom on sheaves. \end{proof}
    
    \section{Case Study: Car Rental}
    Similar variant of transportation problem is used to solve the feasibility of different types of requests that a car rental service provider can simultaneously satisfy. 
    
    In general, the service provider owns a fleet of cars belonging to different models (brand/build). Since, at least few cars are always under maintenance/service, number of cars available for renting varies over days for every model. We represent $m_{db}$ as the number of cars of model $b$ available on $d^{th}$ day. Further $B$ represents list of all models which the service provider owns, so, $b \in B$.
    
    A typical request is represented as $r_i (d_i, \tau_i, j_i, n_i)$ where,
    $r_i$ is the request of $i^{th}$ customer;
    $d_i$ and $\tau_i$ are the staring day of the request and the number of days for which cars are to be rented such that $1 \leq d_i + \tau_i - 1 \leq T$;
    $j_i$ is models of cars requested such that $j_i \subset B$;
    and, $n_i$ is the total number of cars requested. Also, let's say $R$ is the set of all the requests $r_i (d_i, \tau_i, j_i, n_i)$ and $T$ is the total Time Period under consideration. Finally, let's say $R_d$ is the set of all the request which spans across $d^{th}$ day.
    
    \begin{case}No car is under maintenance and cars can be rented for only one day
    \end{case} 
    We begin with a simplified scenario where no car is under maintenance at any point in time. Hence, $m_{db}$ can be simply written as $m_{b}$. Further in this case, we assume that cars are rented for just one day, making $\tau_i = 1$ for all requests. We would lift these restrictions in subsequent sections.
    
    Since this scenario doesn't have requests spanning over multiple days, it doesn't require equal flow analysis and can be solved independently for every day. Let's say $x_{ibd}$ is the number of cars allocated to $i^{th}$ request from brand $b$ on $d^{th}$ day. So, all the requests can be accepted iff,
    
    \begin{equation}
    \begin{array}{rrclcl}
    \displaystyle \exists &X \\
    \textrm{s.t.} & \sum_{b \in B} x_{ibd}  & \geq & n_i && \forall i \in R_d \\
    & x_{ibd}  & = & 0 && \forall (i,b,d) \in \mathfrak{F} \\
    & \sum_{i \in R} x_{ibd}  & \leq & m_b && \forall b \in B \\
    & x_{ibd} & \geq & 0 \\
    & x_{ibd} & \in & Z 
    \end{array}
    \end{equation}
    
    Here second equation represents the forbidden arcs $\mathfrak{F}$ introduced in section 2 and are defined here as:
    $$\mathfrak{F}=\{(i,b,d) \mid i \in R_d; b \not\in j_i \}$$
    
    \begin{case}No car is under maintenance and cars can be rented for multiple consecutive days
    \end{case}
    
    Let's say $x_{ibd}$ is the number of cars allocated to $i^{th}$ request from brand $b$ on $d^{th}$ day, such that, $d_i \leq d \leq d_i + \tau_i - 1$. In this scenario, we need to establish equal flow constraints as same set of cars need to be allocated across all requested days for a particular request. We can use $Class(x_{ib})$, introduced in section 1, to represent all the arcs $(i,b,d)$ that are equal to $x_{ib}$ across all $d$, such that, $d_i \leq d \leq d_i + \tau_i$ - 1. Hence, all requests can be accepted iff,
    
    \begin{equation}
    \begin{array}{rrclcl}
    \displaystyle \exists &X \\
    \textrm{s.t.} & \sum_{b \in B} x_{ibd}  & \geq & n_i && \forall i \in R; d: d_i \leq d \leq d_i + \tau_i - 1\\
    & x_{ibd}  & = & 0 && \forall (i,b,d) \in \mathfrak{F}\\
    & \sum_{i \in R_d} x_{ibd}  & \leq & m_b && \forall b \in B; d: 1 \leq d \leq T \\
    & x_{ibd} & = & x_{ib} && \forall (i,b,d) \in Class(x_{ib})\\
    & x_{ibd} & \geq & 0 \\
    & x_{ibd} & \in & Z 
    \end{array}
    \end{equation}
    
    Here forbidden arcs $\mathfrak{F}$ are defined as:
    $$\mathfrak{F}=\{(i,b,d) \mid i \in R; \text{ either } b \not\in j_i \text{ or } d \not\in [d_i, d_i + \tau_i - 1]\}$$
    
    \begin{case} \label{cs3} Some cars are under maintenance and cars can be rented for multiple consecutive days
    \end{case} 
    
    In the most generalized scenario, we take the case where $m_b$ is not constant over days and hence needs to be represented as $m_{bd}$ to reflect the supply of brand $b$ on $d^{th}$ day. Since same set of cars are supposed to be allocated to a particular request across all the requested days, replacing $m_b$ with $m_{bd}$ in equation 11 doesn't solve the problem directly. Let's say cars with serial number S1 and S2 are available on day 1 and cars with serial number S2 and S3 are available on day 2. Although, in this case, $m_{bd} = 2$ for $d = 1, 2$, there is only $1$ car with serial number S2 that can be given to a request asking for cars on both days. 
    
    This \textit{unequibound} problem is solved by using \textit{variable splitting} method described in section 4.1. There can be a set of cars for every brand/build which are available throughout $T$, so their numbers can be represented as $m_{bT}$. Number of cars of a particular model $b$ which are available on day $d$ but are not available throughout $T$ can be written as $m_{bd^{'}}$. Hence we have:
    $$m_{bT} + m_{bd^{'}} = m_{bd}$$
    Inventories which are available for a set of consecutive days $d$ such that $d_i \leq d \leq d_i + \tau_i - 1$, but are not available throughout $T$, are given by $m_{bd_i\tau_i}$. It should be noted that $m_{bd_i\tau_i} \leq m_{bd^{'}}$ where $d_i \leq d \leq d_i + \tau_i - 1$.
    
    Finally, we also split the variable $x_{ibd}$ into $x_{ibd}^+$ and $x_{ibd}^-$, where $x_{ibd}^+$ is supplied from common pool of $m_{bT}$, whereas $x_{ibd}^-$ is supplied from $m_{bd_i\tau_i}$. Here all requests are accepted iff,
    
    \begin{equation}
    \begin{array}{rrclcl}
    \displaystyle \exists &X \\
    \textrm{s.t.} & \sum_{b \in B} x_{ibd}^+ + x_{ibd}^-  & \geq & n_i && \forall i \in R; d: d_i \leq d \leq d_i + \tau_i - 1\\
    & x_{ibd}^+, x_{ibd}^-  & = & 0 && \forall (i,b,d) \in \mathfrak{F}\\
    & \sum_{i \in R_d} x_{ibd}^+  & \leq & m_{bT} && \forall b \in B; d: 1 \leq d \leq T\\
    & \sum_{i \in R_d} x_{ibd}^-  & \leq & m_{bd^{'}} && \forall b \in B; d: 1 \leq d \leq T\\
    & x_{ibd}^-  & \leq & m_{bd_i\tau_i} && \forall i \in R; b \in B; d: d_i \leq d \leq d_i + \tau_i - 1\\
    & x_{ibd}^+ & = & x_{ib}^+ && \forall (i,b,d) \in Class(x_{ib}^+)\\
    & x_{ibd}^- & = & x_{ib}^- && \forall (i,b,d) \in Class(x_{ib}^-)\\
    & x_{ibd}^+, x_{ibd}^- & \geq & 0 \\
    & x_{ibd}^+, x_{ibd}^- & \in & Z 
    \end{array}
    \end{equation}
    Again forbidden arcs $\mathfrak{F}$ are defined as earlier:
    $$\mathfrak{F}=\{(i,b,d) \mid i \in R; \text{ either } b \not\in j_i \text{ or } d \not\in [d_i, d_i + \tau_i - 1]\}$$

    \bibliographystyle{plain}
    \bibliography{multiincharge.bib}
    \end{document}